\documentclass[11pt]{article}

\usepackage[final]{acl}

\usepackage{times}
\usepackage{latexsym}
\usepackage{amsfonts}
\usepackage{amsmath}
\usepackage{amssymb}
\usepackage{amsthm}
\usepackage{booktabs}

\usepackage{ltablex} 
\usepackage{caption}

\usepackage[T1]{fontenc}

\usepackage[utf8]{inputenc}

\usepackage{microtype}

\usepackage{inconsolata}

\usepackage[final]{graphicx}
\usepackage[dvipsnames]{xcolor}
\usepackage[most]{tcolorbox}
\usepackage{amsthm}
\usepackage{subcaption}
\usepackage{booktabs}
\usepackage[table]{xcolor}
\newcommand{\Dir}{\mathrm{Dir}}
\newcommand{\Cat}{\mathrm{Cat}}

\newtheorem{theorem}{Theorem}
\newtheorem{lemma}{Lemma}
\newtheorem{assumption}{Assumption}
\newtheorem{corollary}{Corollary}
\newtheorem{definition}{Definition}
\newtheorem{proposition}{Proposition}

\tcolorboxenvironment{proposition}{
  enhanced, breakable,
  colback=purple!6,
  colframe=purple!60!black,
  boxrule=0.8pt, arc=2mm,
  left=1mm,right=1mm,top=1mm,bottom=1mm,
}

\tcolorboxenvironment{theorem}{
  enhanced, breakable,
  colback=cyan!6, colframe=cyan!60!black,
  boxrule=0.8pt, arc=2mm,
  left=1mm,right=1mm,top=1mm,bottom=1mm,
}

\tcolorboxenvironment{lemma}{
  enhanced, breakable,
  colback=orange!6, colframe=orange!70!black,
  boxrule=0.8pt, arc=2mm,
  left=1mm,right=1mm,top=1mm,bottom=1mm,
}

\tcolorboxenvironment{assumption}{
  enhanced, breakable,
  colback=yellow!10, colframe=yellow!60!black,
  boxrule=0.8pt, arc=2mm,
  left=1mm,right=1mm,top=1mm,bottom=1mm,
}

\tcolorboxenvironment{corollary}{
  enhanced, breakable,
  colback=green!7, colframe=green!50!black,
  boxrule=0.8pt, arc=2mm,
  left=1mm,right=1mm,top=1mm,bottom=1mm,
}

\tcolorboxenvironment{definition}{
  enhanced, breakable,
  colback=blue!5, colframe=blue!55!black,
  boxrule=0.8pt, arc=2mm,
  left=1mm,right=1mm,top=1mm,bottom=1mm,
}

\newcommand{\ie}{\textit{i}.\textit{e}.,\ }

\newcommand{\rparagraph}[1]{\vspace{1.2mm}\noindent\textbf{#1.}}

\title{Demystifying Multi-Agent Debate: The Role of Confidence and Diversity}

\newcommand{\camid}{{\includegraphics[scale=0.018]{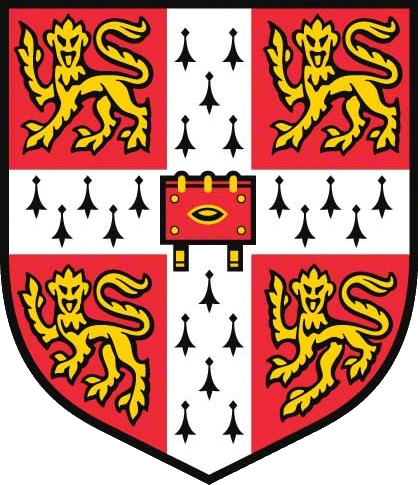}}}
\newcommand{\sheid}{{\includegraphics[scale=0.039]{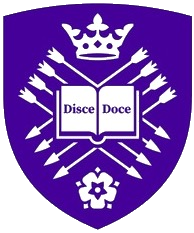}}}
\newcommand{\camemailadress}[1]{\href{mailto:#1@cam.ac.uk}{#1}}
\newcommand{\sheemailadress}[1]{\href{mailto:#1@sheffield.ac.uk}{#1}}

\author{
Xiaochen Zhu\thanks{Equal contribution. Our code is available at: \url{https://github.com/SpaceHunterInf/DMAD}}\textsuperscript{\camid} 
Caiqi Zhang\footnotemark[1]\textsuperscript{\camid}
Yizhou Chi \textsuperscript{\camid} 
Tom Stafford\textsuperscript{\sheid} 
Nigel Collier\textsuperscript{\camid} 
Andreas Vlachos\textsuperscript{\camid} \\
  $^\camid$University of Cambridge~\,~ \textsuperscript{\sheid}University of Sheffield \\
  \{\camemailadress{xz479}, \camemailadress{cz391}, \camemailadress{yc697}, \camemailadress{nhc30},
  \camemailadress{av308}\}@cam.ac.uk \\
  \sheemailadress{t.stafford}@sheffield.ac.uk
}

\begin{document}
\maketitle
\begin{abstract}
Multi-agent debate (MAD) is widely used to improve large language model (LLM) performance through test-time scaling, yet recent work shows that vanilla MAD often underperforms simple majority vote despite higher computational cost. Studies show that, under homogeneous agents and uniform belief updates, debate preserves expected correctness and therefore cannot reliably improve outcomes. Drawing on findings from human deliberation and collective decision-making, we identify two key mechanisms missing from vanilla MAD: (i) diversity of initial viewpoints and (ii) explicit, calibrated confidence communication. We propose two lightweight interventions. First, a diversity-aware initialisation that selects a more diverse pool of candidate answers, increasing the likelihood that a correct hypothesis is present at the start of debate. Second, a confidence-modulated debate protocol in which agents express calibrated confidence and condition their updates on others’ confidence. We show theoretically that diversity-aware initialisation improves the prior probability of MAD success without changing the underlying update dynamics, while confidence-modulated updates enable debate to systematically drift to the correct hypothesis. Empirically, across six reasoning-oriented QA benchmarks, our methods consistently outperform vanilla MAD and majority vote. Our results connect human deliberation with LLM-based debate and demonstrate that simple, principled modifications can substantially enhance debate effectiveness.
\end{abstract}

\section{Introduction}

Multi-agent debate (MAD) has rapidly become a popular technique for improving large language models' (LLMs) reasoning and test-time scaling (TTS) \cite{liang2024encouraging,du2023improving,yang2025revisiting,guo2024large}. 
Inspired by studies of group deliberation, MAD operationalises the idea that independent agents exchanging reasons should, in principle, outperform individual sampling or simple voting \cite{mercier2011humans,moshman1998collaborative}.

\begin{figure}[t]
    \centering
    \includegraphics[width=0.95\columnwidth]{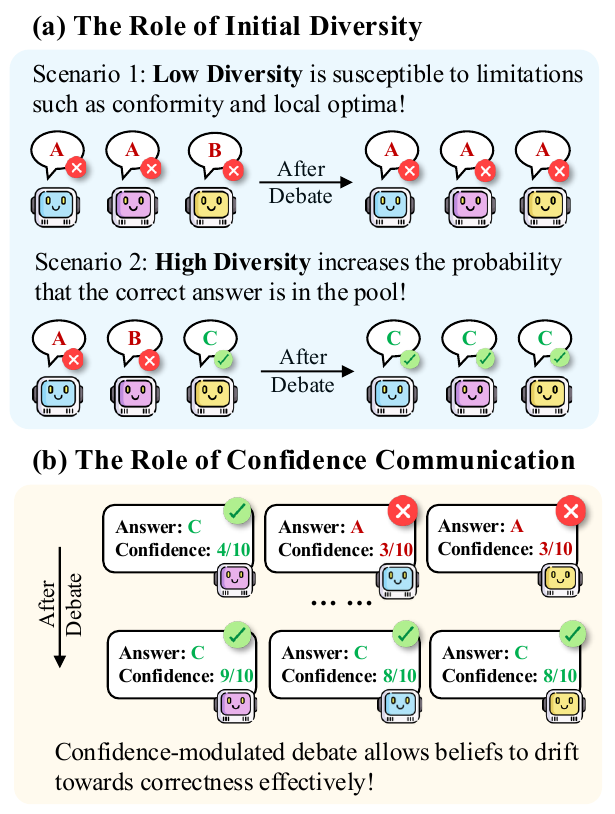}
    \caption{ Illustration of two human-inspired ingredients for effective MAD: (a) diverse initial answers increase the chance the correct hypothesis is present; (b) explicit confidence sharing enables confidence-weighted updates that steer beliefs toward the correct answer. }
    \label{fig:main}

\vspace{-3mm}

\end{figure}

Despite growing enthusiasm and substantial token budgets devoted to multi-agent debate (MAD), its benefits are far from guaranteed \cite{choi2025debate,wu2025can}. In practice, vanilla MAD often underperforms a simple majority vote despite incurring substantially higher computational cost. \citet{choi2025debate} formally characterise this limitation by showing that, under homogeneous agents and unweighted belief updates, debate preserves expected correctness over time.  Mathematically, such behaviour is described as a martingale over belief trajectories. In other words, while agents may exchange arguments, the probability of converging to the correct answer does not systematically increase.

Previous work in cognitive science has identified several factors that affect the success of human deliberation:

\begin{enumerate}
    \item \textbf{Diversity of initial viewpoints.} Human group deliberation benefits from heterogeneous priors and perspectives \cite{moshman1998collaborative}. Exposure to differing ideas 
    increases the search space over possible solutions and thus benefits the deliberation \cite{smaldino2024maintaining}. Empirical work shows that the diversity of positions raised during discussion is positively correlated with performance gains \cite{karadzhov2024effect}.
    \item \textbf{Confidence communication and alignment.} In human discussion, participants do not only share answers; they also communicate uncertainty, negotiate and align their confidence expressions, and use these signals to weight others' contributions when forming collective decisions \cite{bahrami2010optimally,dezecache2022democratic,mercier2017enigma}.
\end{enumerate}

These two mechanisms are largely absent in standard MAD implementations. Firstly, agents are sampled from the same or similar models without regard for diversity, a problem exacerbated by post-training and alignment that can induce diversity collapse \cite{padmakumar2023does,kirk2023understanding}. Secondly, unlike humans who convey confidence through tone, prosody, and facial expression \cite{guyer2021paralinguistic}, LLMs can only communicate confidence in an explicitly verbalised manner. Such signals are rarely requested or used in current MAD protocols, and LLMs are known to exhibit over-confident expressions even when they are wrong \cite{sun2025large}. 

Recent work has started to analyse these failure modes and the role of diversity in MAD. \citet{estornell2024multi} show that when agents lack diversity, debate dynamics can quickly become static, causing the procedure to collapse back to the initial majority. \citet{wynn2025talk} and \citet{wu2025can} empirically validate that, in such settings, MAD often provides little or no gain over majority vote. However, these works mostly diagnose the problem; they do not offer 
mechanisms that can reliably increase diversity and thereby improve debate outcomes. In parallel, several studies attempt to incorporate numerical confidence expressions into debates \cite{yoffe2024debunc,lin2025enhancing} and report promising improvements. Yet their implementations typically rely on post-hoc calibration and focus only on having agents state their confidence, while overlooking how agents interpret and use others' confidence, which is a central aspect of human group decision-making \cite{zarnoth1997social, fusaroli2012coming}.

In this paper, we examine MAD through the lens of these two missing ingredients: \textbf{diversity} and \textbf{confidence}. By analysing debate dialogues, we show that initial answer diversity correlates with performance gains with statistical significance, and that MAD is more effective on harder datasets where such diversity arises naturally. Following this, we introduce two simple interventions. First, a \textbf{diversity-aware initialization} that selects a more diverse subset of candidate answers as the initial debate pool, improving performance without additional training. Second, a \textbf{confidence-modulated debate protocol} in which agents express calibrated confidence and condition their updates on others’ confidence. Crucially, we view them as complementary mechanisms acting on different stages of deliberation: diversity initialisation affects the initial belief distribution (\ie what hypotheses are present), while confidence affects the aggregation dynamics (\ie how each contribution influences others). We provide theoretical guarantees showing that diversity increases the prior probability of the debate success, while confidence-weighted updates break the martingale limitation and allow beliefs to drift toward correctness. Empirically, our methods consistently outperform vanilla MAD and majority vote across six reasoning-oriented question-answering (QA) benchmarks. Conceptually, this work integrates human deliberation theory with LLM-based MAD, offering a principled explanation and practical guidance for designing more effective multi-agent systems.

\section{Related Work}

\rparagraph{Multi-Agent Debate}
Human group deliberation often improves collective performance because discussion allows individuals to pool information, detect mistakes of each other, and benefit from diverse cognitive perspectives \cite{kerr2004group}. MAD aims to replicate these advantages with LLMs, revising low-quality responses, reducing hallucinations, and eliciting reasoning that may be overlooked by a single model \cite{chan2023chateval,liang2024encouraging,estornell2024multi}. 
The standard MAD protocol asks each agent to answer the same question and then revise their response given previous agents' arguments \cite{yang2025revisiting}. While MAD sometimes yields modest gains over single-model sampling, a growing body of work shows that it often fails to inherit the benefits of human debate. Several studies report that MAD can perform worse than simple majority vote, especially when debates extend over multiple rounds \cite{wynn2025talk,wu2025can}. \citet{wu2025can} argue that, unlike human groups, LLM agents struggle to recognise and disregard low-quality contributions, causing noise to propagate rather than dissipate. Most notably, \citet{choi2025debate} show that homogeneous LLM agents produce debate dynamics that behave empirically like a martingale over agents' expressed beliefs; they further prove that MAD cannot surpass the accuracy of majority vote in expectation. These findings reveal a fundamental discrepancy between human deliberation and LLM debate and motivate reconsideration of the assumptions and protocols underlying MAD.

\rparagraph{Diversity in Collective Problem Solving}
Diversity is a central factor in the success of human deliberation. Psychological work shows that heterogeneous viewpoints help mitigate conformity pressure and ``groupthink,'' where groups converge to majority opinions regardless of correctness \cite{janis1972victims,asch1955opinions}. Maintaining a diversity of candidate solutions and exploring disagreement improves group reasoning outcomes \cite{karadzhov2024effect,smaldino2024maintaining}. 
MAD exhibits similar vulnerabilities. Without sufficient diversity, debates collapse to the initial majority and fail to recover from incorrect early answers \cite{zhu2025conformity,estornell2024multi}. The problem is exacerbated by properties of contemporary LLMs: post-training alignment reduces sampling diversity \cite{kirk2023understanding,padmakumar2023does, hu2025navigating}, and models often exhibit sycophantic tendencies toward user prompts \cite{sharma2023towards}. Recent analyses of MAD suggest that greater diversity in initial reasoning is associated with improved performance \cite{wynn2025talk,liang2024encouraging,estornell2024multi}. However, this line of work primarily attributes diversity to heterogeneity across model families, parameterisations, or prompting styles. As a result, it remains unclear what fundamentally governs answer diversity or how it can be reliably increased in a simple, training-free manner. For example, \citet{liang2024encouraging} and \citet{choi2025debate} encourage diversity by prompting agents with different personas, while \citet{estornell2024multi} propose a KL-divergence–based pruning strategy to remove redundant answers during debate. These approaches alter prompting or debate structure but do not directly increase the diversity of the initial answer pool, nor do they explicitly target the prior probability of including a correct hypothesis.

\rparagraph{Confidence, Calibration, and Uncertainty Communication}
In human deliberation, confidence sharing is crucial for effective collaboration: individuals communicate uncertainty, align and adjust confidence expressions through interaction, and use others' confidence cues to weight arguments when forming joint decisions \cite{bahrami2010optimally,fusaroli2012coming,mercier2020not}. Such communication is facilitated by paralinguistic signals, tone, prosody, facial expression \cite{guyer2021paralinguistic}, which are unavailable in text-only LLM settings.  Although recent work has sought to teach LLMs to proactively express uncertainty during generation \citep{zhang2025reinforcement, yang-etal-2025-uncle, yang-etal-2025-logu}, these methods have not yet been systematically studied in debate or multi-turn settings \citep{zhang2026multiturn}.
Several recent studies instead introduce numerical confidence into MAD.  \citet{lin2025enhancing} and \citet{yoffe2024debunc} use post-hoc calibration methods to assign confidence scores to agents' answers and use these scores as additional signals during debate. While these approaches yield limited improvements, they diverge from psychological theories that emphasise aligned confidence communication and shared metacognition. Moreover, they overlook the gap between confidence expression and perception in LLMs: models are known to be systematically overconfident in their verbalised confidence judgments \cite{sun2025large}, and there is no guarantee that LLM agents interpret or appropriately weight others' confidence. As a result, current confidence-based approaches do not fully capture the mechanisms that make human deliberation successful.

\section{Formulation}

Following prior theoretical analyses of multi-agent debate (MAD), we adopt the framework of \citet{choi2025debate} and consider a controlled setting with homogeneous LLM agents performing question answering (QA). This abstraction isolates the structural properties of debate dynamics without conflating them with model heterogeneity or explicit role specialisation. 

Let $\{a_1, \dots, a_N\}$ denote $N$ agents. For an input question $x \in \mathcal{X}$, each agent samples an answer
\begin{equation}
    y_i \sim a_i(x), \qquad a_i(x) := f(x;\theta_i),
\end{equation}
where $f$ is the conditional distribution defined by an LLM with parameters $\theta_i$.  
In the homogeneous setting, all agents share identical parameters ($\theta_i = \theta$), and stochasticity arises solely from sampling.

\rparagraph{Majority Vote}
Given an initial pool of sampled responses $\{y_i\}_{i=1}^N$, the majority vote baseline selects
\begin{equation}
        V(\{y_i\}_{i=1}^N)
        = \arg\max_{y} \, \mathrm{freq}(y).
\end{equation}
This corresponds to test-time ensembling without interaction and provides an important empirical baseline for vanilla MAD.

\rparagraph{Multi-Agent Debate}
A debate consists of $T$ rounds of iterative revision. At round $t=1$, each agent independently samples an initial answer $y_{i,1} \sim a_i(x)$. For each subsequent round $t>1$, every agent observes all the answers produced by other agents and itself in the previous round:
\begin{equation}
    \mathcal{R}_t = \{\, y_{j,t-1} \mid j = 1,\dots,N \,\}.
\end{equation}
Each agent then revises its answer by applying an answer-update operator
\begin{equation}
    y_{i,t} = \mathcal{D}(x, \mathcal{R}_t),
\end{equation}
where $\mathcal{D}$ is instantiated as a prompting-based revision operator executed by the LLM.

After $T$ rounds, the ensemble-level MAD output is obtained by majority vote over the terminal responses:
\begin{equation}
    y_{\mathrm{MAD}} = V(\{y_{i,T}\}_{i=1}^N).
\end{equation}

\section{Methodology}

Incentivized by the role of diversity and confidence in human deliberation, we incorporate diversity-aware initialization and confidence modulated debate as additional features on top of vanilla MAD.

\subsection{Diversity-aware Initialization}

Studies of human group deliberation suggest that considering a more diverse set of candidate solutions is correlated with higher collective answer correctness \cite{karadzhov2024effect}. One proposed explanation is that broader exploration of the solution space increases the likelihood that at least one high-quality hypothesis is discovered, rather than converging prematurely to suboptimal consensus \cite{smaldino2024maintaining}. We adopt this intuition as a design principle for initializing multi-agent debate.

Formally, given a set of answers $S$, we define its diversity as the number of distinct answers:
\begin{equation}
    \mathrm{div}(S) = |\mathrm{unique}(S)|.
\end{equation}
Vanilla MAD initializes debate by independently sampling $N$ answers $\{y_{i,1}\}_{i=1}^N$. Instead, we first sample a larger pool of $N_{\text{cand}} \ge N$ candidate answers $ \{y^{(1)}, \dots, y^{(N_{\text{cand}})}\}$,
and select a subset of size $N$ to initialise the debate by solving
\begin{equation}
    S_{\mathrm{div}}
    = \underset{S \subseteq \{1,\dots,N_{\text{cand}}\},\, |S| = N}{\operatorname{argmax}}
    \mathrm{div}(S).
\end{equation}
We apply a simple, training-free greedy approximation that iteratively selects the candidate with the largest marginal contribution to diversity until $N$ answers are chosen. Compared to vanilla MAD, this procedure incurs additional test-time cost by sampling $N_{\text{cand}}$ rather than $N$ initial answers, but requires no model training or architectural modification. Unlike approaches that prune similar responses within a single debate round \cite{estornell2024multi}, our method preserves the standard MAD protocol and increases the probability that the initial debate state contains at least one correct hypothesis.

\subsection{Confidence-Modulated Debate}
\label{sec:confidence-modulated debate}
\citet{bahrami2010optimally} show that when debate participants can directly communicate their uncertainty, joint decisions improve. In human deliberation, such uncertainty cues are often conveyed through paralinguistic signals (e.g., facial expressions and tone) \cite{guyer2021paralinguistic}. However, these cues are unavailable in text-based LLM debates, where confidence must be expressed verbally. Moreover, estimating confidence for black-box models such as LLMs is itself non-trivial. While confidence can be inferred indirectly through sampling-based uncertainty estimates \citep{zhang-etal-2024-luq, zhang2025atomic} or linguistic markers \citep{yang-etal-2025-logu, yang-etal-2025-uncle}, we focus on verbalised numerical confidence, as it is the most direct signal agents can exchange during debate and can be explicitly calibrated \citep{zhang2025reinforcement}.

To incorporate confidence into debate, we extend each agent's output at round $t$ to
\begin{equation}
    y^{conf}_{i,t} = (y_{i,t}, \, w_{i,t}),
\end{equation}
where $w_{i,t} \in \{0, \dots, 10\}$ is a discrete confidence score indicating the agent's self-assessed certainty, where 0 indicates the agent is unsure about its answer and 10 indicates total confidence. The debate-update operator becomes
\begin{equation}
    y^{w}_{i,t}
        = \mathcal{D}^{w}(x, \mathcal{R}_{t}^w),
\end{equation}
where $\mathcal{R}_{t}^w$ contains each agent's answer and confidence.

We use the term \emph{modulated} to emphasise that confidence does not change the content of an agent’s answer, but instead scales its influence during aggregation, analogous to how modulation alters the strength of a carrier signal without changing the underlying message \cite{crecraft2002analog}. However, leveraging such confidence signals introduces two \textbf{key challenges}. First, verbally reported scores are often miscalibrated: LLMs may express high confidence even when incorrect \cite{sun2025large}. Second, even with calibrated scores, agents may fail to \emph{use} confidence effectively during interaction---\ie they may ignore confidence cues or update their beliefs inappropriately. We therefore (i) train agents to express calibrated numerical confidence, and (ii) teach agents to perceive and exploit confidence signals during debate via reinforcement learning.

\rparagraph{Calibrated Confidence Expression}
Rather than relying on post-hoc calibration \cite{lin2025enhancing,yoffe2024debunc}, we directly train the model to generate calibrated confidence scores during debate using reinforcement learning (RL).

Similar to \citet{zhang2025reinforcement} and \citet{stangel2025rewarding}, we use a binary cross-entropy style log-based reward encouraging alignment between correctness and expressed confidence. Let $z_{i,t} \in \{0,1\}$ denote the correctness indicator of $y_{i,t}$ and map $w_{i,t}$ to $(0,1)$, we adopt the following reward:
\begin{equation}
\begin{split} 
r^{\mathrm{conf}}(y_{i,t}, w_{i,t}) = \frac{\lambda}{R_{\max}}
      [z_{i,t}\, \log(w_{i,t}) \\
      \;+\; (1 - z_{i,t})\, \log(1 - w_{i,t}) ].
\end{split} 
\end{equation}
$\lambda$, $R_{max}$ are scaling coefficients. This encourages the model to assign higher confidence to correct predictions and lower confidence to incorrect ones.  

\rparagraph{Confidence Perception and Usage}
Expression alone does not guarantee that agents use confidence appropriately during debate. As emphasised in \citet{fusaroli2012coming}, effective deliberation requires confidence alignment, where individuals adjust their reasoning strategies based on others' confidence signals. To train LLM agents to perceive and use confidence, we augment the reward with correctness signal and apply RL to the debate-update operator:
\begin{equation}
\label{eq: perception loss}
    r^{\mathrm{total}}
        = \lambda_1 \, z_{i,t}
        \;+\; \lambda_2 \, r^{\mathrm{conf}}(y_{i,t}, w_{i,t}),
\end{equation}
where $\lambda_1, \lambda_2$ control the trade-off between accuracy and calibration. In this case, the learned policy is encouraged to explicitly condition on others' confidence. This captures both calibrated expression and active alignment and usage of confidence, two components necessary for mirroring human deliberation dynamics.

\section{Theoretical Analysis}
\label{sec:theory}

We now analyse how does our interventions improve over the vanilla MAD dynamics of \citet{choi2025debate}. Unlike prior work that studies diversity or confidence in isolation, we analyse them as complementary mechanisms acting on different stages of the same deliberative process. We work within their Dirichlet–categorical model (DCM) for MAD, and state only the ingredients needed for our results; a full description is given in the Appendix \ref{app:DCM}.

Let $A = \{1,\dots,K\}$ be the finite answer set and assume, without loss of generality, that option $1$ is the unique correct answer. At debate round $t$, agent $i$ is parameterised by a Dirichlet vector
\begin{equation}
  \boldsymbol{\alpha}_{i,t}
  =
  \bigl(\alpha^{(1)}_{i,t},\dots,\alpha^{(K)}_{i,t}\bigr)
  \in \mathbb{R}^K_{>0},
\end{equation}
and generates its answer $y_{i,t} \in A$ by
\[
  \boldsymbol{\theta}_{i,t} \sim \Dir(\boldsymbol{\alpha}_{i,t}),
  \qquad
  y_{i,t} \mid \boldsymbol{\theta}_{i,t} \sim \Cat(\boldsymbol{\theta}_{i,t}).
\]
We write
\begin{equation}
  p_{i,t}
  :=
  \Pr(y_{i,t} = 1 \mid \boldsymbol{\alpha}_{i,t})
  =
  \frac{\alpha^{(1)}_{i,t}}{\sum_{k=1}^K \alpha^{(k)}_{i,t}}
\end{equation}
for agent $i$'s belief on the correct option at round $t$.

In the given DCM model, agents update their Dirichlet parameters by adding one count for each agent’s answer in the previous round. Let $\mathcal{F}_t$ denote the filtration generated by all belief parameters up to round $t$. In the homogeneous, fully connected setting where all agents share the same prior and observe the same multiset of answers, \citet{choi2025debate} show that
\begin{equation}
  \mathbb{E}\!\left[p_{i,t} \mid \mathcal{F}_{t-1}\right]
  = p_{i,t-1},
\end{equation}
so $\{p_{i,t}\}_{t > 0}$ is a martingale. This means during the debate, the expected belief at the next round always equals to the current belief. In expectation, vanilla MAD neither helps nor hurts the performance. 

\subsection{Diversity Improves What Is Debated}
\label{sec:theory-diversity}

Diversity-aware initialisation affects the support of the debate by increasing the likelihood that the initial answer pool contains at least one correct 
hypothesis. Importantly, this intervention operates entirely at initialisation and does not modify the subsequent debate dynamics.

\begin{proposition}[Diversity-aware initialization improves prior success]
\label{prop:diversity-main}
Let $A_T$ be the event that debate outputs the correct answer at the final round $T$ under the unweighted DCM dynamics. Let $S$ denote the number of distinct informative hypotheses in the initial answer pool (e.g., the number of distinct options that are sampled at least once). Suppose that the conditional success probability $\mathbb{P}(A_T \mid
S=s)$ is nondecreasing in $s$, and that the distribution of $S$ under our diversity-aware initialiser first-order stochastically dominates that under random i.i.d.\ sampling. Then
\[
\mathbb{P}(A_T \mid \text{diverse init})
\;\ge\;
\mathbb{P}(A_T \mid \text{random init}).
\]
\end{proposition}

Proposition~\ref{prop:diversity-main} formalises that diversity-aware
initialization improves MAD by shifting the distribution of the initial debate state toward pools with broader hypothesis coverage. Conditional on this state, the unweighted debate dynamics remain a martingale, but the prior probability that debate begins with a useful hypothesis is strictly increased. A full proof is provided in Appendix~\ref{app:diversity-effect}.

\subsection{Confidence Improves How Debate Aggregates}

Confidence-modulated debate changes how information is aggregated. Answers that are both high-confidence and positively correlated with correctness receive more weight in the Dirichlet updates, causing the expected belief in the correct answer to drift upward over rounds instead of remaining flat as in the unweighted martingale case. Consider the DCM debate model described above and extend it so that, at each round, agents exchange answers together with scalar confidence scores and update their Dirichlet parameters using confidence-weighted counts.

\begin{theorem}[Confidence-weighted debate yields a submartingale]
\label{thm:confidence-main}

Assume that: (i) agents are homogeneous and fully connected; and (ii) confidence is positively correlated with correctness, in the sense that higher-confidence answers are, on average, more likely to be
correct than lower-confidence ones. Let $p_{i,t}$ denote agent $i$'s belief on the correct answer at round $t$, and let $\mathcal{F}_t$ be the filtration generated by the debate history. Then the belief process $\{p_{i,t}\}_{t\ge 0}$ becomes a strict submartingale:
\[
\mathbb{E}\!\left[p_{i,t} \mid \mathcal{F}_{t-1}\right]
\;\ge\;
p_{i,t-1},
\]
with strict inequality on a set of positive probability. 
\end{theorem}

Conceptually speaking, confidence-modulated debate breaks the martingale symmetry of vanilla MAD and, in expectation, strictly improves correctness over rounds. Whereas a martingale preserves the expected belief in the correct answer, a submartingale has increasing expected belief, implying systematic progress toward the correct hypothesis. Appendix~\ref{app:confidence-weighted} defines the confidence-weighted Dirichlet-categorical update, and Appendix~\ref{app:submartingale} proves that this yields a (strict) submartingale.

\begin{table*}[t]
\small
\centering
\renewcommand{\arraystretch}{1.1}
\setlength{\tabcolsep}{5pt}

\scalebox{0.95}{
\begin{tabular}{lcccccccc}
\toprule
\textbf{Model Strategy} & \textbf{GSM8K} & \textbf{CSQA} & \textbf{HellaSwag} & \textbf{MMLU-FL} & \textbf{MMLU-PM} & \textbf{MMLU-All} & \textbf{GPQA} & \textbf{ARC-C} \\
\midrule
\multicolumn{9}{l}{\textit{Qwen-2.5-7B-Instruct}} \\
\midrule
Single Model
& \cellcolor[HTML]{FFFFFF}0.908 & \cellcolor[HTML]{FFFFFF}0.784 & \cellcolor[HTML]{FFFFFF}0.752 & \cellcolor[HTML]{FFFFFF}0.547
& \cellcolor[HTML]{FFFFFF}0.765 & \cellcolor[HTML]{FFFFFF}0.735 & \cellcolor[HTML]{FFFFFF}0.298 & \cellcolor[HTML]{FFFFFF}0.805 \\
Debate $5\times5$
& \cellcolor[HTML]{FFFFFF}0.847 & \cellcolor[HTML]{FFFFFF}0.727 & \cellcolor[HTML]{FFFFFF}0.698 & \cellcolor[HTML]{E5F0E7}0.571
& \cellcolor[HTML]{FFFFFF}0.734 & \cellcolor[HTML]{FFFFFF}0.727 & \cellcolor[HTML]{F3FAF4}0.339 & \cellcolor[HTML]{FBFEFC}0.882 \\
\rowcolor{black!7}
Majority Vote
& \cellcolor[HTML]{FFFFFF}0.908 & \cellcolor[HTML]{FFFFFF}0.808 & \cellcolor[HTML]{FFFFFF}0.762 & \cellcolor[HTML]{FFFFFF}0.548
& \cellcolor[HTML]{FFFFFF}0.806 & \cellcolor[HTML]{FFFFFF}0.764 & \cellcolor[HTML]{FFFFFF}0.335 & \cellcolor[HTML]{FFFFFF}0.880 \\
\midrule
High Diversity
& \cellcolor[HTML]{E2F6E6}0.920 & \cellcolor[HTML]{C6EFCE}\textbf{0.836} & \cellcolor[HTML]{C6EFCE}0.779 & \cellcolor[HTML]{C6EFCE}0.587
& \cellcolor[HTML]{C6EFCE}0.822 & \cellcolor[HTML]{FFFFFF}0.743 & \cellcolor[HTML]{FFFFFF}0.334 & \cellcolor[HTML]{F3FAF4}0.884 \\
HD + Learn2Agg
& \cellcolor[HTML]{FFFFFF}0.893 & \cellcolor[HTML]{FFFFFF}0.776 & \cellcolor[HTML]{D2ECD6}0.777 & \cellcolor[HTML]{C6EFCE}0.603
& \cellcolor[HTML]{FFFFFF}0.805 & \cellcolor[HTML]{FFFFFF}0.752 & \cellcolor[HTML]{CFEFD6}0.353 & \cellcolor[HTML]{E7F7EA}0.889 \\
Confidence
& \cellcolor[HTML]{FBFEFC}0.910 & \cellcolor[HTML]{FBFEFC}0.810 & \cellcolor[HTML]{C6EFCE}\textbf{0.782} & \cellcolor[HTML]{C6EFCE}0.587
& \cellcolor[HTML]{E2F6E6}0.818 & \cellcolor[HTML]{F3FAF4}\textbf{0.769} & \cellcolor[HTML]{D2ECD6}0.350 & \cellcolor[HTML]{C6EFCE}\textbf{0.911} \\
HD + Confidence
& \cellcolor[HTML]{C6EFCE}\textbf{0.932} & \cellcolor[HTML]{CDEFD4}0.828 & \cellcolor[HTML]{C9F0D0}0.781 & \cellcolor[HTML]{C6EFCE}\textbf{0.619}
& \cellcolor[HTML]{C9F0D0}\textbf{0.825} & \cellcolor[HTML]{F3FAF4}\textbf{0.769} & \cellcolor[HTML]{CDEFD4}\textbf{0.361} & \cellcolor[HTML]{E7F7EA}0.890 \\
\midrule
\multicolumn{9}{l}{\textit{Llama-3.1-8B-Instruct}} \\
\midrule
Single Model
& \cellcolor[HTML]{FFFFFF}0.797 & \cellcolor[HTML]{FFFFFF}0.703 & \cellcolor[HTML]{FFFFFF}0.676 & \cellcolor[HTML]{FFFFFF}0.503
& \cellcolor[HTML]{FFFFFF}0.798 & \cellcolor[HTML]{FFFFFF}0.701 & \cellcolor[HTML]{FFFFFF}0.266 & \cellcolor[HTML]{FFFFFF}0.780 \\
Debate $5\times5$
& \cellcolor[HTML]{FFFFFF}0.850 & \cellcolor[HTML]{FFFFFF}0.688 & \cellcolor[HTML]{FFFFFF}0.667 & \cellcolor[HTML]{E2F6E6}0.544
& \cellcolor[HTML]{FFFFFF}0.772 & \cellcolor[HTML]{FFFFFF}0.680 & \cellcolor[HTML]{CDEFD4}0.316 & \cellcolor[HTML]{FFFFFF}0.802 \\
\rowcolor{black!7}
Majority Vote
& \cellcolor[HTML]{FFFFFF}0.870 & \cellcolor[HTML]{FFFFFF}0.718 & \cellcolor[HTML]{FFFFFF}0.686 & \cellcolor[HTML]{FFFFFF}0.532
& \cellcolor[HTML]{FFFFFF}0.809 & \cellcolor[HTML]{FFFFFF}0.711 & \cellcolor[HTML]{FFFFFF}0.290 & \cellcolor[HTML]{FFFFFF}0.832 \\
\midrule
High Diversity
& \cellcolor[HTML]{FFFFFF}0.863 & \cellcolor[HTML]{FBFEFC}0.721 & \cellcolor[HTML]{D2ECD6}\textbf{0.701} & \cellcolor[HTML]{CDEFD4}0.548
& \cellcolor[HTML]{FFFFFF}0.805 & \cellcolor[HTML]{FFFFFF}0.687 & \cellcolor[HTML]{C6EFCE}0.321 & \cellcolor[HTML]{C6EFCE}0.877 \\
HD + Learn2Agg
& \cellcolor[HTML]{CDEFD4}\textbf{0.891} & \cellcolor[HTML]{C6EFCE}0.747 & \cellcolor[HTML]{FFFFFF}0.677 & \cellcolor[HTML]{C9F0D0}0.551
& \cellcolor[HTML]{FFFFFF}0.797 & \cellcolor[HTML]{FFFFFF}0.700 & \cellcolor[HTML]{C6EFCE}0.350 & \cellcolor[HTML]{FFFFFF}0.832 \\
Confidence
& \cellcolor[HTML]{E2F6E6}0.882 & \cellcolor[HTML]{CDEFD4}0.734 & \cellcolor[HTML]{FFFFFF}0.685 & \cellcolor[HTML]{CDEFD4}0.548
& \cellcolor[HTML]{E7F7EA}0.818 & \cellcolor[HTML]{E7F7EA}\textbf{0.720} & \cellcolor[HTML]{C6EFCE}\textbf{0.352} & \cellcolor[HTML]{C6EFCE}\textbf{0.882} \\
HD + Confidence
& \cellcolor[HTML]{CDEFD4}0.890 & \cellcolor[HTML]{C6EFCE}\textbf{0.752} & \cellcolor[HTML]{FBFEFC}0.688 & \cellcolor[HTML]{E5F0E7}\textbf{0.555}
& \cellcolor[HTML]{C6EFCE}\textbf{0.842} & \cellcolor[HTML]{FBFEFC}0.713 & \cellcolor[HTML]{C6EFCE}\textbf{0.352} & \cellcolor[HTML]{DCEEE0}0.854 \\
\bottomrule
\end{tabular}
}
\caption{Benchmark performance for Qwen-2.5-7B-Instruct and Llama-3.1-8B-Instruct across major datasets. Cell background indicates improvement (green) relative to the \colorbox{black!7}{Majority Vote} baseline within each model group; color intensity scales with the magnitude of the change. Best results per model group are \textbf{bolded}. We specifically look into MMLU Formal Logic (FL) and Professional Medicine (PM) to keep consistent with \citet{choi2025debate}.}
\label{tab:main}
\vspace{-3mm}

\end{table*}

\begin{table}[h]
\footnotesize
\centering
\begin{tabular}{lcc}
\midrule
Model & r & p \\
\midrule
Qwen-2.5-7B-Instruct  & 0.070 & <0.001 \\
Llama-3.1-8B-Instruct & 0.019 & 0.002  \\
\bottomrule
\end{tabular}
\caption{Pearson correlation (r) and significance levels (p) for initial answer diversity vs.\ debate performance gain (n=26{,}624).}
\label{tab: div vs. gain}
\vspace{-3mm}
\end{table}

\section{Experiment Setup}

We first conduct a series of experiments to study the dynamics of vanilla MAD and provide additional analysis under the setting of \citet{choi2025debate}. Specifically, we initialise 5 homogeneous agents using the same language model, and run debates for 5 turns. To obtain a comprehensive understanding, we evaluate MAD across a wide range of tasks using multiple language models.

\rparagraph{Datasets}
We cover a series of QA datasets with varying levels of difficulty. Our in-domain datasets include GSM8K \cite{cobbe2021gsm8k}, CommonsenseQA \cite{talmor-etal-2019-commonsenseqa}, HellaSwag \cite{zellers2019hellaswag}, and MMLU \cite{hendryckstest2021}, which contains 57 sub-tasks across diverse knowledge domains. We perform confidence-modulated debate training on these datasets. To assess the generalization capabilities of our methods, we additionally evaluate on GPQA-Main \cite{rein2024gpqa} and ARC-Challenge \cite{allenai:arc}, which we treat as out-of-domain (OOD) benchmarks.

\rparagraph{Language Models and Prompts}
We evaluate Llama-3.1-8B-Instruct \cite{dubey2024llama} and Qwen-2.5-7B-Instruct \cite{qwen2.5} for both vanilla MAD and our trained variants. In the MAD setting, we elicit chain-of-thought (CoT) reasoning \cite{wei2022chain} by prompting the model to output its reasoning steps followed by its final answer. To incorporate confidence, we additionally prompt the model to produce a numerical confidence value. Detailed prompts are provided in Appendix~\ref{app:prompts}.

\rparagraph{Sampling and Training Schema}
For agent initialisation and all inference steps, we apply temperature $1.0$ and nucleus sampling with $p=0.9$ \cite{holtzman2019curious}. For diversity injection, we set $N_{\text{cand}} = 10$. For confidence expression calibration and perception training, we use the GRPO algorithm combined with LoRA adaptation \cite{shao2024deepseekmath,hu2022lora}. Further details on the sampling strategy and training configuration are provided in the Appendix \ref{app: sampling and training}. We evaluate our methods against single model, majority vote, and majority vote after debate as baselines.

\section{Results \& Analysis}
\begin{figure}[t]
  \centering
  
  \begin{subfigure}[b]{0.24\textwidth}
    \centering
    \includegraphics[width=\linewidth, trim=15cm 10cm 15cm 55cm, clip]{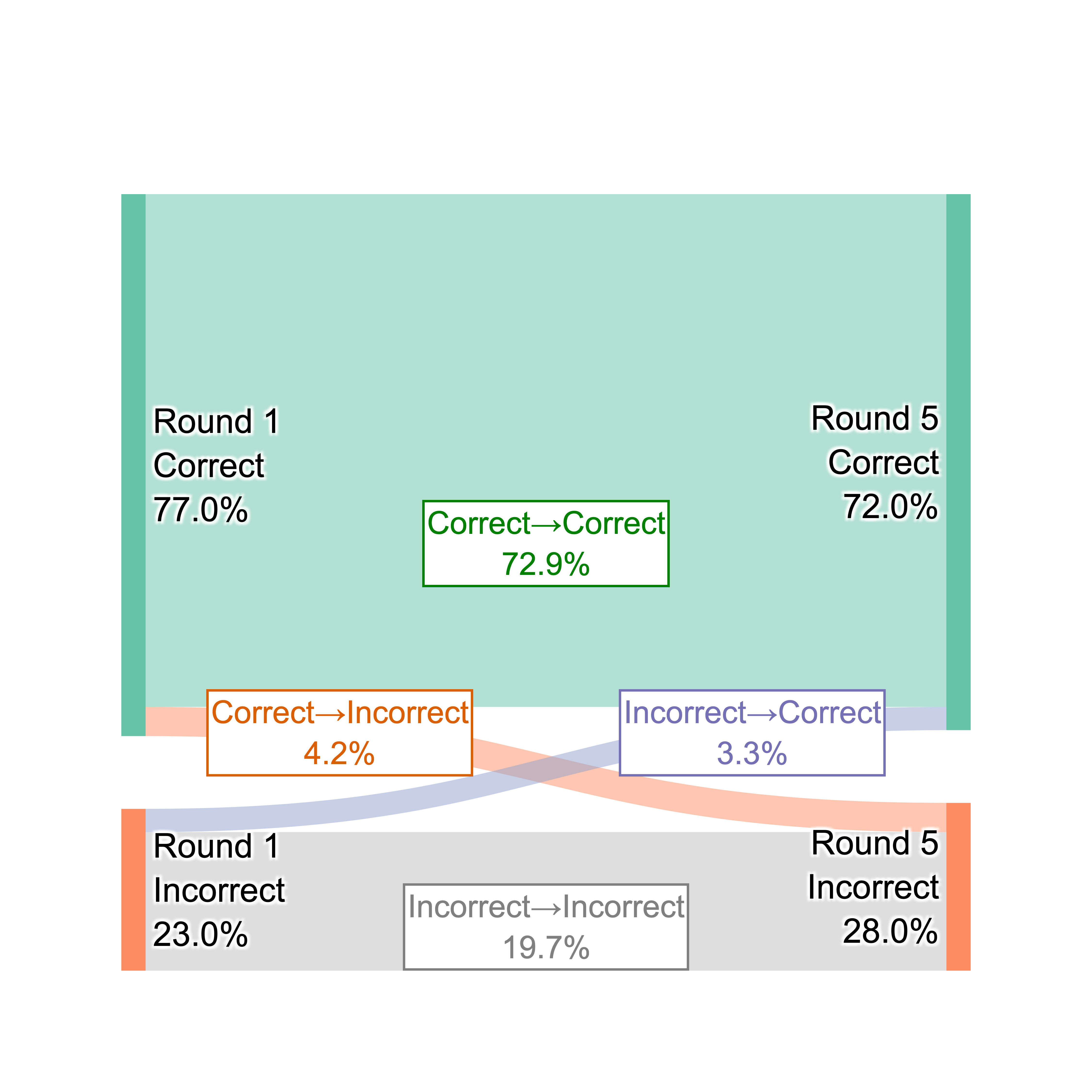}
    \caption{Qwen Vanilla}
    \label{fig:3}
  \end{subfigure}\hfill
  \begin{subfigure}[b]{0.24\textwidth}
    \centering
    \includegraphics[width=\linewidth, trim=15cm 10cm 15cm 55cm, clip]{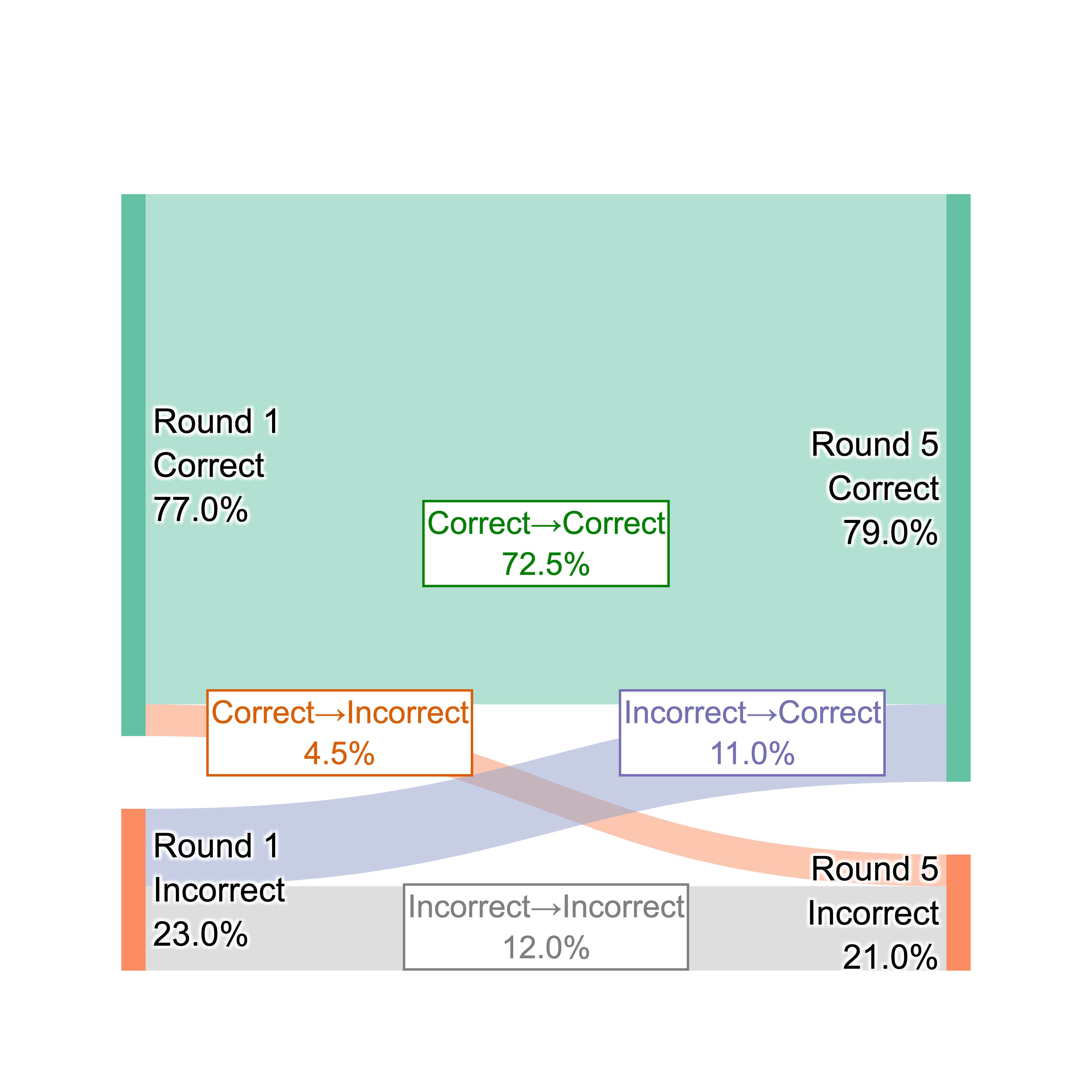}
    \caption{Qwen Confidence}
    \label{fig:4}
  \end{subfigure}
  \caption{Sankey plots of correctness transitions from round 1 to round 5 (Qwen): The confidence-modulated debate (b) shows improvement on answer revision (blue), indicating better belief drift to correctness. }
  \label{fig:flow_distribution}
\vspace{-3mm}

\end{figure}
\rparagraph{Diversity alone improves MAD performance} We observe that vanilla MAD performs similarly or worse than simple majority vote, echoing findings reported in prior work \cite{choi2025debate}. However, further analysis reveals several meaningful patterns. As shown in Table~\ref{tab: div vs. gain}, even without any intervention, initial answer diversity in vanilla MAD is significantly correlated with final answer accuracy. While the correlation is weak, considering its consistency and direction in addition to the diversity collapse of the instruct model, they together suggest that diversity exerts a systematic effect in vanilla MAD, motivating mechanisms that explicitly leverage this signal. This also mirrors findings in human group deliberation, where diversity promotes collective accuracy \cite{karadzhov2024effect}.

We additionally study how diversity interacts with task difficulty. We define the difficulty as a model’s accuracy on a given dataset. Figure~\ref{fig:diversity_corr} shows that MAD naturally produces a more diverse initial answer pool on harder datasets (\ie low accuracy). Consequently, as illustrated in Figure~\ref{fig:perf_gain}, MAD tends to be more beneficial on more challenging tasks. These results provide an empirical explanation for when MAD helps and when it fails.

\begin{figure}[t]
  \centering
  \begin{subfigure}[t]{0.49\columnwidth}
    \centering
    \includegraphics[width=\linewidth]{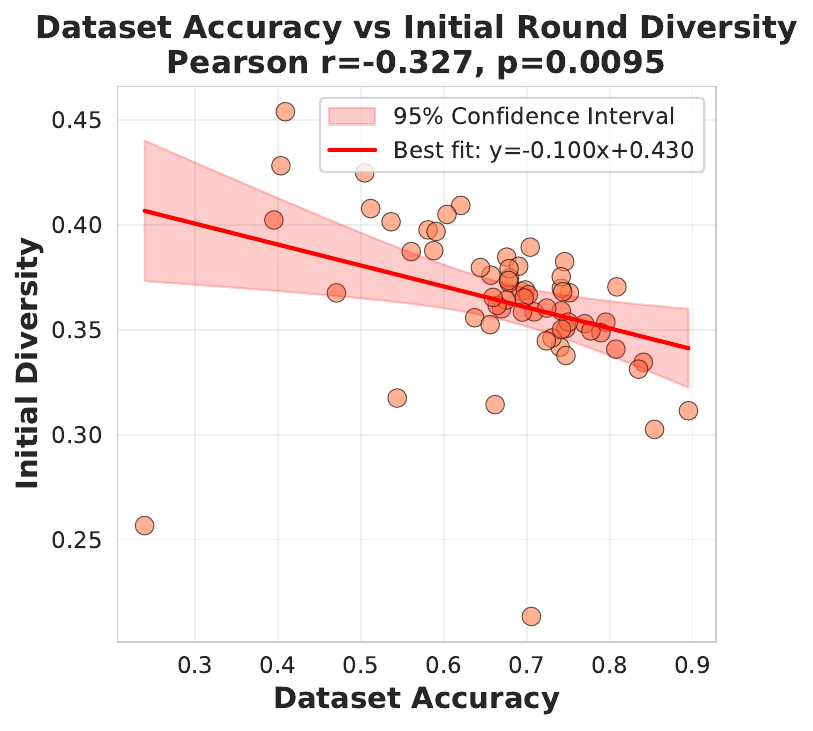}
    \caption{Diversity vs. Difficulty}
    \label{fig:diversity_corr} 
  \end{subfigure}\hfill
  \begin{subfigure}[t]{0.49\columnwidth}
    \centering
    \includegraphics[width=\linewidth]{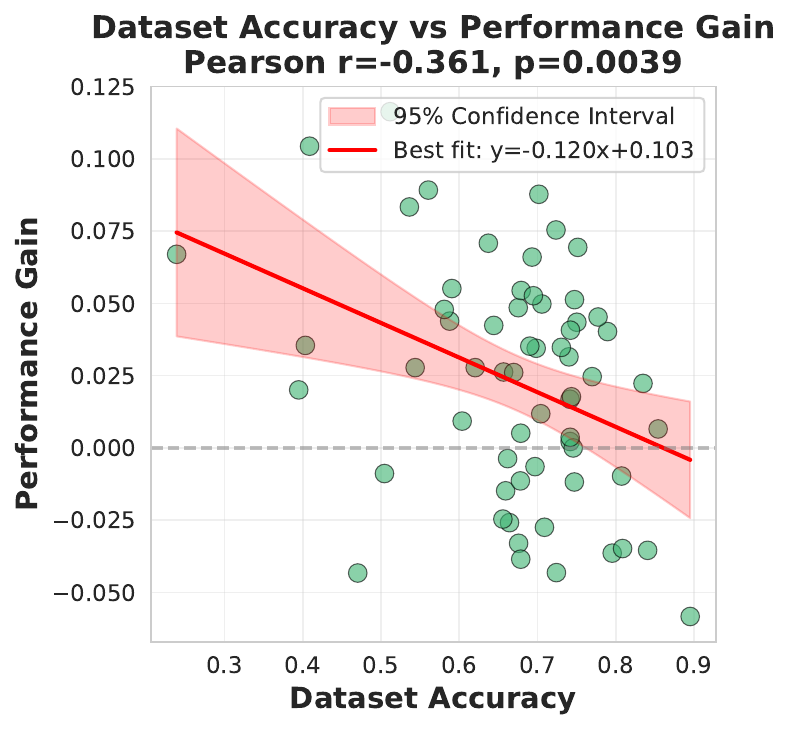}
    \caption{Diversity vs. Performance Gain}
    \label{fig:perf_gain} 
  \end{subfigure}

  \caption{\textbf{(a)} Correlation between initial answer diversity and dataset difficulty (measured by model accuracy) over 62 datasets (including 57 MMLU sub-tasks). \textbf{(b)} Relationship between dataset difficulty and debate performance gain. MAD tends to yield larger gains on more difficult datasets.}
  \label{fig:main_analysis}
\vspace{-3mm}

\end{figure}

As shown in Table~\ref{tab:main}, across both models, our diversity-aware initialisation consistently improves accuracy on nearly all datasets and surpasses simple majority vote. To quantify this effect, we directly measure (i) the number of unique answers and (ii) Pass@5 in the initial round (i.e., whether any agent initially produces a correct answer). As shown in Table~\ref{tab: div prior}, with diversity-aware initialisation, both metrics increase across the board, demonstrating that our intervention raises the prior probability of MAD success by ensuring more debates begin with at least one correct hypothesis. This directly supports the theoretical conclusion of Proposition~\ref{prop:diversity-main}.

\begin{table}[h]
\footnotesize
\centering
\begin{tabular}{lcc}
\toprule
\multicolumn{3}{c}{\textbf{(a) Overall Pass@5}} \\
\midrule
Model & Random & Diverse \\
\midrule
Qwen-2.5-7B-Instruct  & 0.7921 & 0.9097 \\
Llama-3.1-8B-Instruct & 0.7423 & 0.9026 \\
\midrule
\multicolumn{3}{c}{\textbf{(b) Number of Unique Answers@5}} \\
\midrule
Model & Random & Diverse \\
\midrule
Qwen-2.5-7B-Instruct  & 1.45 & 1.61 \\
Llama-3.1-8B-Instruct & 1.92 & 2.32 \\
\bottomrule
\end{tabular}
\caption{Comparison of Pass@5 and answer diversity across models and initialisation regimes.}
\label{tab: div prior}
\vspace{-3mm}
\end{table}

\rparagraph{Confidence-modulated debate improves aggregation} Our confidence-modulated debate includes an additional training phase in which agents learn to express calibrated confidence and to use others’ confidence when revising their answers. To isolate the effect of confidence, we compare against an ablation that removes the confidence reward and simply trains a model to select the best answer via aggregation (Learn2Agg), similar to \citet{zhao2025majority}. Even without diversity-aware initialisation, the confidence-modulated debate consistently outperforms simple aggregation across all evaluated datasets, including the two OOD benchmarks. This confirms that incorporating confidence signals provides a reliable mechanism for improving debate quality, consistent with the theoretical prediction of Theorem~\ref{thm:confidence-main}. We further analyse how the dynamics change with in the debate. We conducted a controlled experiment using the initial turn output from vanilla MAD to initialise the confidence modulated debate. As shown in Figure \ref{fig:flow_distribution}, confidence modulation increases correction of initially incorrect answers and stabilises correct ones. This indicates that confidence-modulated debate improved belief propagation with better drift to correctness, aligning with our theoretical analysis of Theorem~\ref{thm:confidence-main}.

\section{Conclusion}

We revisit multi-agent debate (MAD) through the lens of human group deliberation and recent theory showing that vanilla MAD induces a martingale over agents’ beliefs. Drawing on psycholinguistic and socio-cognitive insights, we identify two missing ingredients in standard LLM debates: diversity of initial viewpoints and explicit, calibrated confidence communication. We introduce a diversity-aware initialization that increases the probability that debate starts with a correct hypothesis, and a confidence-modulated debate protocol in which agents express and use calibrated confidence. Theoretically, we show that diversity improves prior probability of MAD success without affecting martingale dynamics, while confidence-weighted updates break the martingale symmetry and yield a submartingale over correctness. Empirically, our methods consistently outperform vanilla MAD and majority vote across six benchmarks. Overall, our results demonstrate that simple, principled changes inspired by human deliberation can substantially improve LLM multi-agent debate.

\section*{Limitations}

Our work has several limitations. First, our theoretical analysis builds on the Dirichlet--categorical debate model and homogeneity assumptions of \citet{choi2025debate}. While this abstraction enables clear insights into debate dynamics, it does not fully capture the complexity of real LLM behavior. In particular, the DCM model reduces belief updates to simple counting, ignoring that agents may receive different inputs or contextual positions, which can induce effects such as position bias \cite{wang2024large}. Moreover, the martingale analysis abstracts away settings in which LLMs may refine their own beliefs through self-reflection or internal reasoning \cite{huang2023large}.

Second, although our diversity-aware initialization is training-free and effective, it relies on a heuristic selection of diverse outputs from a fixed candidate pool. We do not claim this strategy to be optimal, and more principled or adaptive objectives for promoting diversity may further improve debate outcomes \cite{he2025rewarding,anschel2025group}.

Third, our confidence-modulated debate assumes that expressed confidence is positively correlated with correctness after calibration. While our training procedure improves this alignment empirically, confidence miscalibration remains a known challenge for LLMs, and failures in confidence estimation could reduce the effectiveness of confidence-weighted aggregation.

Fourth, our experiments focus on homogeneous agents and fully connected debate graphs. Extending the framework to heterogeneous models, asymmetric communication topologies, or larger agent populations may introduce additional dynamics not captured in this work.

Finally, we evaluate MAD primarily on English language question-answering benchmarks. While these tasks are standard for studying reasoning and debate, future work is needed to assess whether our findings generalize to more open-ended, interactive, or real-world multi-agent settings.

\section*{Ethics Statement}

Our research adheres to strict ethical guidelines. We verified the licenses of all the software and datasets used in this study to ensure full compliance with their terms. No privacy concerns have been identified. We have conducted a thorough assessment of the project and do not anticipate any further risks. We only used AI assistance for grammar checking when writing the paper.

\section*{Acknowledgement}
We thank Zhaobo Han, PhD candidate in pure mathematics at UCLA, for reviewing our proofs. Andreas Vlachos is supported by the DARPA program SciFy.

\bibliography{custom}

\appendix

\clearpage

\section{Prompts}
\label{app:prompts}

Here we provide the prompt we used for MAD. We use the prompt in Table \ref{apptab:initial prompt} to elicit the initial pool of answers for the debate. For the subsequent rounds, we use the prompt in Table \ref{apptab:debate prompt} to encourage debate. For a single agent, the other agents' responses in the previous round will be concatenated together and presented together with its own previous response. For settings not incorporating confidence, the confidence related descriptions are simply removed. All the interactions are done using chat template.

\begin{table}[h!]
\centering
\begin{tcolorbox}[colback=blue!5!white, colframe=blue!50!black, title = {Initial Round Prompt with Confidence}, fontupper=\footnotesize, fonttitle=\footnotesize]

Answer the question. Think step by step first, at the end of your reasoning, provide your final answer and confidence level, ranging from 0 to 10, where 0 means no confidence at all and 10 means complete confidence. \\
\\
Output format: \\
<reasoning> YOUR DETAILED REASONING HERE </reasoning> \\
\\
<answer> YOUR FINAL ANSWER </answer> \{extra\_info\} \\
\\
<confidence> INTEGER </confidence>
\\

\{question\}

\end{tcolorbox}
\caption{Initial Round Prompt with Confidence.}
\label{apptab:initial prompt}
\end{table}

\begin{table}[h!]
\centering
\begin{tcolorbox}[colback=blue!5!white, colframe=blue!50!black, title = {MAD Prompt with Confidence}, fontupper=\footnotesize, fonttitle=\footnotesize]

You are revising your answer after reviewing other agents’ reasoning.

Question: \{question\}

Other agents’ responses:
\{other\_agents\}

Your previous reasoning and answer:
\{reasoning\}

Instructions: \\
- Reflect on how others reasoned. \\
- You may revise your answer if someone’s reasoning provides stronger evidence. \\
- However, if you believe all of them missed something important, propose a better or alternative answer — clearly explain why. \\
- Be concise and clear. \\
- Update your confidence level to reflect how certain you are now. Confidence level is ranging from 0 to 10, where 0 means no confidence at all and 10 means complete confidence. \\

Think step by step first, at the end of your reasoning, provide your final answer and confidence level in the following format:

<reasoning> YOUR DETAILED REASONING HERE </reasoning> \\
<answer> YOUR FINAL ANSWER </answer> \{extra\_info\} \\
<confidence> INTEGER </confidence>

\end{tcolorbox}
\caption{MAD Prompt with Confidence.}
\label{apptab:debate prompt}
\end{table}

\section{Sampling and Training Schema}
\label{app: sampling and training}
All experiments are conducted on NVIDIA A100 GPUs (80GB), with a total consumption of approximately 1{,}500 GPU hours. We use \texttt{vLLM} (v0.11.0) for all inference. Both training and inference are performed in \texttt{float16} precision. For all fine-tuning stages, we employ LoRA with rank 64.

\subsection{Confidence Expression Training}

We first initialize the model with a small number of supervised fine-tuning (SFT) steps to improve format adherence for confidence expression. To construct gold labels, we subsample 5k data points from the combined training datasets and use self-consistency \cite{wang2022self} to derive numerical confidence targets. We then proceed with reinforcement learning (RL) training on a subsampled set of 10k difficult examples from the training data. Table \ref{apptab:calibration-result} demonstrate that trained models exhibits well-calibrated verbalised confidence. Model also maintains a similar accuracy demonstrating the training does not introduce new knowledge that causing data leakage.

\begin{table}[h]
\footnotesize
\centering
\begin{tabular}{lcccc}
\toprule
\multicolumn{5}{c}{\textbf{Llama-3.1-8B-Instruct}} \\
\midrule
Method & Brier & ECE & AUROC & Acc \\
\midrule
Vanilla & 0.217 & 0.257 & 0.644 & 0.491 \\
SC      & 0.051 & 0.157 & 0.852 & N/A   \\
VC      & 0.069 & 0.167 & 0.763 & 0.502 \\
\midrule
\multicolumn{5}{c}{\textbf{Qwen-2.5-7B-Instruct}} \\
\midrule
Method & Brier & ECE & AUROC & Acc \\
\midrule
Vanilla & 0.212 & 0.217 & 0.576 & 0.567 \\
SC      & 0.078 & 0.156 & 0.855 & N/A   \\
VC      & 0.046 & 0.153 & 0.754 & 0.578 \\
\bottomrule
\end{tabular}
\caption{Calibration and performance metrics for Llama-3.1-8B-Instruct and Qwen-2.5-7B-Instruct on hard sub-sampled datasets. Lower is better for Brier score and ECE; higher is better for AUROC and accuracy. All values rounded to three significant digits.}
\label{apptab:calibration-result}
\end{table}

\subsection{Confidence Perception and Usage Training}

Since the input–output format between debate rounds is identical, the debate process can be viewed as Markovian, with each turn conditioned only on agents’ responses from the previous round. We therefore apply GRPO to single-turn inputs and outputs using the combined correctness and confidence rewards defined in Equation~\eqref{eq: perception loss}.

During training, we observed reward hacking behavior, where agents stopped engaging with others’ reasoning or confidence signals and instead output only a final answer. To mitigate this, we introduce a simple engagement constraint that encourages constructive interaction by rewarding the presence of discourse cues (e.g., ``agent'', ``I agree'', ``but I think'', ``convinced'', ``skeptical''). This heuristic helps ensure that agents actively reference and respond to others during debate.

\begin{table}[t]
\centering
\footnotesize

\begin{tabular}{ll}
\toprule
\multicolumn{2}{c}{\textbf{GRPO Confidence Calibration Setup}} \\
\midrule
LoRA & ($r{=}64$, $\alpha{=}32$; q/k/v/o proj) \\
Seq. Length & 2048 \\
\midrule
Epochs & 1 \\
Batch Size & 4 ($\times$2 grad acc.) \\
Learning Rate & $5\times10^{-6}$ \\
\midrule
Generations / Prompt & 8 \\
Sampling & $T{=}1.0$, top-$p{=}0.9$ \\s
KL Coefficient & $\beta{=}0.01$ \\
\bottomrule
\end{tabular}
\caption{Training configuration for confidence calibration.}
\end{table}

\begin{table}[t]
\centering
\footnotesize
\begin{tabular}{ll}
\toprule
\multicolumn{2}{c}{\textbf{GRPO Multi-Agent Debate Setup}} \\
\midrule
LoRA & ($r{=}64$, $\alpha{=}32$; q/k/v/o proj) \\
\midrule
Epochs & 1 \\
Batch Size & 4 ($\times$4 grad acc.) \\
Learning Rate & $5\times10^{-5}$ \\
\midrule
Generations / Prompt & 8 \\
Sampling & $T{=}1.0$, top-$p{=}1.0$ \\
KL Coefficient & $\beta{=}0.01$ \\
Max Completion Length & 1024 \\
\midrule
Confidence Used & Yes \\
Reward Scales & Correct: 10,\; Conf: 3,\; \\ & Engage: 5 \\
Wrong-Format Penalty & $-30$ \\
\bottomrule
\end{tabular}
\caption{GRPO configuration for multi-agent debate training.}
\end{table}

\section{What We Tried but Did Not Work}

\rparagraph{Increasing temperature does not increase answer diversity}
In early experiments, we attempted to promote diversity by increasing the sampling temperature. However, as shown in Table \ref{apptab: temperature vs diversity}, raising the temperature from $1.0$ to $1.2$ did not yield more diverse answers or improved MAD performance. We attribute this to the fact that higher temperature primarily induces diversity in \emph{reasoning paths} rather than in final answers. For instance, Qwen often produced the same answer expressed through different languages (e.g., Chinese or Korean) or stylistic variations. A higher temperature also harms instruction following and lead to worse output quality. This observation suggests that naive temperature scaling is insufficient for inducing answer-level diversity and motivates the need for more targeted diversity-aware initialization strategies.

\begin{table}[h]
\centering
\footnotesize
\begin{tabular}{lcc}
\toprule
\multicolumn{3}{c}{\textbf{Qwen-2.5-7B-Instruct}} \\
\midrule
Temperature & Unique Answers@5 & Accuracy \\
\midrule
$t=0.8$ & 1.18 & 0.722 \\
$t=1.0$ & 1.45 & 0.753 \\
$t=1.2$ & 1.42 & 0.733 \\
\bottomrule
\end{tabular}
\caption{Effect of temperature on answer diversity and average accuracy for Qwen-2.5-7B-Instruct. Values rounded to three significant digits.}
\label{apptab: temperature vs diversity}
\end{table}

\rparagraph{Training on harder data improves confidence calibration}
For confidence expression training, we initially subsampled 12k data points uniformly from the training datasets. In practice, the model struggled to converge under this setting. We attribute this to the model’s already high accuracy on these datasets: during GRPO training, many trajectories resulted in correct answers with maximal confidence ($10$), leading to a collapse in reward diversity and poor sampling efficiency. To address this issue, we manually curated a harder subset of data on which model accuracy was closer to 50\%. Training on this subset substantially improved convergence and led to more stable confidence calibration.

\onecolumn
\section{Proofs}
\label{app:confidence-martingale}

In this appendix we show that both our interventions of diversity injection and confidence-modulated debate could benefit the result of the vanilla MAD which is essentially a martingale process as established in \cite{choi2025debate}. For diversity injection, we are not changing the dynamics of the martingale process, but increases the prior probability of success. For confidence-modulated debate, the incorporation of a confidence score that is positively correlated with correctness breaks martingale process into submartingale. 

\subsection{Baseline DCM Debate Model}
\label{app:DCM}

We briefly recall the Bayesian model of multi-agent debate used in
\citet{choi2025debate}. Let $A=\{1,\dots,K\}$ denote the finite answer
set and let answer $1$ be the unique correct option. There are $N$
agents indexed by $i\in\{1,\dots,N\}$. At debate round $t$, agent $i$ is
parameterized by a Dirichlet belief vector
\begin{equation}
  \boldsymbol{\alpha}_{i,t} =
  \bigl(\alpha^{(1)}_{i,t},\dots,\alpha^{(K)}_{i,t}\bigr)
  \in \mathbb{R}^K_{>0},
\end{equation}
and generates its answer $y_{i,t}\in A$ as
\begin{align*}
  \boldsymbol{\theta}_{i,t} &\sim \Dir(\boldsymbol{\alpha}_{i,t}), \\
  y_{i,t} \mid \boldsymbol{\theta}_{i,t}
    &\sim \Cat(\boldsymbol{\theta}_{i,t}).
\end{align*}
The induced marginal probability of choosing option $k$ is
\begin{equation}
  \Pr(y_{i,t}=k \mid \boldsymbol{\alpha}_{i,t})
  = \frac{\alpha^{(k)}_{i,t}}{\sum_{j=1}^K \alpha^{(j)}_{i,t}}.
\end{equation}

We write
\begin{equation}
  p_{i,t}
  := \Pr(y_{i,t}=1 \mid \boldsymbol{\alpha}_{i,t})
   = \frac{\alpha^{(1)}_{i,t}}{\sum_{j=1}^K \alpha^{(j)}_{i,t}}
\end{equation}
for agent $i$'s belief (probability) on the correct option at round $t$.

Let $G$ be the (undirected) debate graph and $\mathcal{N}(i)$ denote the
neighbors of $i$ (possibly including $i$ itself). In the
\emph{unweighted} model, agent $i$ aggregates its neighbors' answers
$\{y_{j,t-1} : j\in\mathcal{N}(i)\}$ into a count vector
$\mathbf{c}_{i,t} = (c^{(1)}_{i,t},\dots,c^{(K)}_{i,t})$, where
\begin{equation}
  c^{(k)}_{i,t} = \sum_{j\in\mathcal{N}(i)}
    \mathbf{1}\{y_{j,t-1}=k\}.
\end{equation}
The Dirichlet parameters are then updated by Bayesian conjugacy:
\begin{equation}
  \boldsymbol{\alpha}_{i,t}
  = \boldsymbol{\alpha}_{i,t-1} + \mathbf{c}_{i,t}.
\end{equation}

Let $\mathcal{F}_{t}$ denote the filtration generated by all belief
parameters up to round $t$,
$\mathcal{F}_{t} := \sigma\bigl(\{\boldsymbol{\alpha}_{i,s} : i\le N,
s\le t\}\bigr)$. In the homogeneous, fully-connected setting where all
agents share the same prior and observe the same multiset of answers,
\citet{choi2025debate} show that
\begin{equation}
  \mathbb{E}\!\left[p_{i,t} \mid \mathcal{F}_{t-1}\right]
  = p_{i,t-1},
\end{equation}
so $\{p_{i,t}\}_{t> 0}$ is a martingale. Intuitively, each neighbour
contributes one \emph{unweighted} ``ball'' to a Pólya-urn-like update,
and symmetry ensures that, in expectation, debate neither helps nor
hurts.

\subsection{Effect of Diversity Injection}
\label{app:diversity-effect}

Recall that the martingale result of \citet{choi2025debate} states that,
conditional on the initial debate state $\mathcal{F}_1$ (the first-round pool
and its arguments), the expected belief does not drift:
\begin{equation}
\mathbb{E}[p_t \mid \mathcal{F}_1] = p_1 \qquad \forall t,
\end{equation}
so the probability that debate outputs the correct answer is determined entirely
by the quality of $\mathcal{F}_1$.

We now formalize how diversity-aware initialization improves this prior success
probability without altering the martingale dynamics.

\begin{theorem}[Effect of diversity-aware initialization]
\label{thm:diversity}
Let $A_T$ be the event that debate outputs the correct answer at the final
round $T$, and let $S \in \mathbb{N}$ denote the number of distinct informative
hypotheses (e.g., distinct correct or partially-correct reasoning chains)
present in the initial pool $\mathcal{F}_1$. Assume:
\begin{enumerate}
  \item[(i)] (\emph{Monotonicity in coverage}) The success probability is
  weakly increasing in $S$, i.e.
  \begin{equation}
    \label{eq:monotone}
    \Pr(A_T \mid S = s)
    \quad \text{is nondecreasing in } s .
  \end{equation}
  \item[(ii)] (\emph{Diversity increases coverage}) Let
  $S_{\mathrm{rand}}$ and $S_{\mathrm{div}}$ denote the values of $S$
  under random sampling and diversity-aware initialization respectively.
  Then $S_{\mathrm{div}}$ first-order stochastically dominates
  $S_{\mathrm{rand}}$, i.e.
  \begin{equation}
    \label{eq:FOSD}
    \Pr(S_{\mathrm{div}} \ge s)
    \;\ge\;
    \Pr(S_{\mathrm{rand}} \ge s)
    \quad \forall s,
  \end{equation}
  with strict inequality for some $s$.
\end{enumerate}
Then the overall probability of eventual correctness is higher under
diversity-aware initialization:
\begin{equation}
\Pr(A_T \mid \text{diverse})
\;\ge\;
\Pr(A_T \mid \text{random}),
\end{equation}
with strict inequality whenever the dominance in~\eqref{eq:FOSD} is strict.
\end{theorem}

\begin{proof}
Define the function
\begin{equation}
f(s) := \Pr(A_T \mid S = s), \qquad s \in \mathbb{N}.
\end{equation}
Assumption~\eqref{eq:monotone} states that $f$ is nondecreasing.

Under a given initialization scheme (random or diverse), we can write the
overall success probability using the law of total probability as
\begin{equation}
\Pr(A_T \mid \text{scheme})
=
\mathbb{E}\big[f(S_{\text{scheme}})\big],
\end{equation}
where $S_{\text{scheme}}$ is the corresponding random variable
($S_{\mathrm{rand}}$ or $S_{\mathrm{div}}$).

By first-order stochastic dominance~\eqref{eq:FOSD}, we know that for any
nondecreasing function $f$,
\begin{equation}
\mathbb{E}\big[f(S_{\mathrm{div}})\big]
\;\ge\;
\mathbb{E}\big[f(S_{\mathrm{rand}})\big],
\end{equation}
with strict inequality if dominance is strict on a set where $f$ is
strictly increasing.

Combining these facts gives
\begin{equation}
\Pr(A_T \mid \text{diverse})
=
\mathbb{E}\big[f(S_{\mathrm{div}})\big]
\;\ge\;
\mathbb{E}\big[f(S_{\mathrm{rand}})\big]
=
\Pr(A_T \mid \text{random}),
\end{equation}
with strict inequality under strict dominance. This is exactly the
claimed result.
\end{proof}

\paragraph{Simple illustration (signal-coverage model).}
As a concrete example, suppose each distinct informative hypothesis in
the initial pool independently has correctness probability $p>0$. Then,
conditional on observing $S=s$ such hypotheses,
\begin{equation}
\Pr(A_T \mid S = s)
= 1 - (1-p)^s,
\end{equation}
which is strictly increasing in $s$, verifying the monotonicity
assumption~\eqref{eq:monotone}. If diversity-aware initialization
increases the effective spread of hypotheses so that
$S_{\mathrm{div}}$ stochastically dominates $S_{\mathrm{rand}}$, then
\begin{equation}
\Pr(A_T \mid \text{diverse})
=
\mathbb{E}\big[1-(1-p)^{S_{\mathrm{div}}}\big]
\;\ge\;
\mathbb{E}\big[1-(1-p)^{S_{\mathrm{rand}}}\big]
=
\Pr(A_T \mid \text{random}),
\end{equation}
in agreement with Theorem~\ref{thm:diversity}.

\paragraph{Relation to the martingale result.}
Note that all improvement arises from a better initial distribution over
$\mathcal{F}_1$. The martingale property of \citet{choi2025debate} states
that, once $\mathcal{F}_1$ is fixed, debate cannot raise the conditional
expectation of correctness in subsequent rounds. Diversity-aware
initialization operates \emph{before} this point: it shifts the
distribution of $\mathcal{F}_1$ toward states that already contain useful
hypotheses, thereby increasing the overall probability of correctness
without altering the debate dynamics themselves.

\subsection{Confidence-Weighted Debate}
\label{app:confidence-weighted}
We now extend this model by allowing agents to emit not only an answer
$y_{i,t}$ but also a scalar confidence $w_{i,t}\in (0,1]$. At each
round $t$, agent $i$ therefore outputs a pair
$(y_{i,t}, w_{i,t})$.

\paragraph{Confidence-weighted counts.}
Given neighbors' outputs $\{(y_{j,t-1},w_{j,t-1}): j\in\mathcal{N}(i)\}$,
agent $i$ constructs a \emph{confidence-weighted} count vector
$\widetilde{\mathbf{c}}_{i,t}\in\mathbb{R}^K_{\ge 0}$ with components
\begin{equation}
  \label{eq:weighted-counts}
  \widetilde{c}^{(k)}_{i,t}
  := \sum_{j\in\mathcal{N}(i)} w_{j,t-1}\,
     \mathbf{1}\{y_{j,t-1}=k\}, \qquad k=1,\dots,K.
\end{equation}
We then update the Dirichlet parameters by adding the weighted counts:
\begin{equation}
  \label{eq:conf-update}
  \boldsymbol{\alpha}_{i,t}
  = \boldsymbol{\alpha}_{i,t-1} + \widetilde{\mathbf{c}}_{i,t}.
\end{equation}
This preserves conjugacy:
\begin{equation}
  \boldsymbol{\theta}_{i,t} \mid \mathcal{F}_{t}
  \sim \Dir(\boldsymbol{\alpha}_{i,t}),
\end{equation}
but the increments are now real-valued rather than integer counts.

For notational convenience, write the total Dirichlet mass as
\begin{equation}
  S_{i,t} := \sum_{k=1}^K \alpha^{(k)}_{i,t},
\end{equation}
and define, for the update from $t-1$ to $t$,
\begin{equation}
  \Delta_{i,t}
  := \sum_{k=1}^K \widetilde{c}^{(k)}_{i,t}
  = \sum_{j\in\mathcal{N}(i)} w_{j,t-1},
\end{equation}
and the confidence-weighted ``mass on the correct answer''
\begin{equation}
  \Delta^{(1)}_{i,t}
  := \widetilde{c}^{(1)}_{i,t}
  = \sum_{j\in\mathcal{N}(i)} w_{j,t-1}\,
    \mathbf{1}\{y_{j,t-1}=1\}.
\end{equation}
Under the confidence-weighted update~\eqref{eq:conf-update}, the belief
of agent $i$ in the correct answer at round $t$ can be written as
\begin{equation}
  \label{eq:p-update-fraction}
  p_{i,t}
  = \frac{\alpha^{(1)}_{i,t-1} + \Delta^{(1)}_{i,t}}
         {S_{i,t-1} + \Delta_{i,t}}.
\end{equation}

\subsection{From Martingale to Submartingale}
\label{app:submartingale}
We now study the conditional expectation of $p_{i,t}$ given
$\mathcal{F}_{t-1}$. The key observation is that, under
\eqref{eq:p-update-fraction}, the next-round belief can be written as a
convex combination of the current belief $p_{i,t-1}$ and a
confidence term.

\begin{lemma}[Convex combination form]
\label{lem:convex-combo}
Assume $\Delta_{i,t} > 0$. Define the realised fraction of confidence
mass that lands on the correct answer as
\begin{equation}
  \hat{q}_{i,t-1}
  := \frac{\Delta^{(1)}_{i,t}}{\Delta_{i,t}}.
\end{equation}
Then, for each agent $i$ and round $t$,
\begin{equation}
  \label{eq:p-convex-combo-pathwise}
  p_{i,t}
  = \bar\lambda_{i,t}\, p_{i,t-1}
    + \bigl(1-\bar\lambda_{i,t}\bigr)\, \hat{q}_{i,t-1},
\end{equation}
where
\begin{equation}
  \bar\lambda_{i,t}
  := \frac{S_{i,t-1}}{S_{i,t-1} + \Delta_{i,t}}
  \in (0,1).
\end{equation}
In particular, $p_{i,t}$ is (pathwise) a convex combination of
$p_{i,t-1}$ and $\hat{q}_{i,t-1}$.
\end{lemma}

\begin{proof}
Starting from~\eqref{eq:p-update-fraction} and using
$\alpha^{(1)}_{i,t-1} = p_{i,t-1} S_{i,t-1}$ and
$\Delta^{(1)}_{i,t} = \hat{q}_{i,t-1}\,\Delta_{i,t}$, we obtain
\begin{equation}
  p_{i,t}
  = \frac{p_{i,t-1} S_{i,t-1} + \hat{q}_{i,t-1}\,\Delta_{i,t}}
         {S_{i,t-1} + \Delta_{i,t}}.
\end{equation}
Define $x := S_{i,t-1}/\Delta_{i,t}$. Then
\begin{equation}
  p_{i,t}
  = \frac{x p_{i,t-1} + \hat{q}_{i,t-1}}{x+1}
  = \frac{x}{x+1}\, p_{i,t-1}
    + \frac{1}{x+1}\, \hat{q}_{i,t-1},
\end{equation}
which is exactly~\eqref{eq:p-convex-combo-pathwise} with
$\bar\lambda_{i,t} = x/(x+1)\in(0,1)$.
\end{proof}

Lemma~\ref{lem:convex-combo} shows that, on each realisation, the new
belief is a convex combination of the previous belief and the realised
confidence-weighted fraction of correct answers among the neighbors. To
understand the effect on the \emph{expected} belief, we now take
conditional expectations and formalise a positive-correlation assumption.

\begin{assumption}[Confidence is positively correlated with correctness]
\label{ass:conf-correlation}
For each agent $j$ and round $t-1$, define
\begin{equation}
  p_{j,t-1}
  := \Pr(y_{j,t-1}=1 \mid \mathcal{F}_{t-1}),
\end{equation}
and let
\begin{equation}
  \rho_{j,t-1}
  := \frac{
      \mathbb{E}\!\left[w_{j,t-1}\,\mathbf{1}\{y_{j,t-1}=1\}
        \mid \mathcal{F}_{t-1}\right]
    }{
      \mathbb{E}\!\left[w_{j,t-1} \mid \mathcal{F}_{t-1}\right]
    }.
\end{equation}
We say confidence is \emph{positively correlated with correctness} if
\begin{equation}
  \rho_{j,t-1} \;\ge\; p_{j,t-1}
  \quad\text{for all $j,t$,}
\end{equation}
with strict inequality on a set of positive probability (for at least
one agent and round). In words, among the mass of answers that receive
high confidence, the fraction that are correct is at least as large as
the agent's marginal accuracy, and strictly larger sometimes.
\end{assumption}

We also define the conditional expectation of the confidence-weighted
correctness fraction:

\begin{definition}[Confidence-weighted expected correctness]
For each agent $i$ and round $t$, define
\begin{equation}
  \label{eq:q-def}
  q_{i,t-1}
  := \frac{
    \mathbb{E}\!\left[\Delta^{(1)}_{i,t} \mid \mathcal{F}_{t-1}\right]
  }{
    \mathbb{E}\!\left[\Delta_{i,t} \mid \mathcal{F}_{t-1}\right]
  }.
\end{equation}
\end{definition}

Under mild regularity conditions (e.g., finite second moments), $q_{i,t-1}$
captures the expected fraction of total confidence mass that lands on the
correct answer.

We can now relate $q_{i,t-1}$ to the neighbour-level
$\rho_{j,t-1}$'s:

\begin{lemma}[Decomposition of confidence-weighted correctness]
\label{lem:q-decomposition}
For each agent $i$ and round $t$,
\begin{equation}
  q_{i,t-1}
  = \sum_{j\in\mathcal{N}(i)} \omega_{j,t-1}\, \rho_{j,t-1},
\end{equation}
where
\begin{equation}
  \omega_{j,t-1}
  := \frac{
       \mathbb{E}[w_{j,t-1}\mid\mathcal{F}_{t-1}]
     }{
       \sum_{\ell\in\mathcal{N}(i)}
       \mathbb{E}[w_{\ell,t-1}\mid\mathcal{F}_{t-1}]
     }
  \in (0,1),
  \qquad \sum_{j\in\mathcal{N}(i)} \omega_{j,t-1} = 1.
\end{equation}
\end{lemma}

\begin{proof}
By definition,
\begin{align*}
  \mathbb{E}[\Delta^{(1)}_{i,t} \mid \mathcal{F}_{t-1}]
  &= \mathbb{E}\!\left[
      \sum_{j\in\mathcal{N}(i)} w_{j,t-1}
        \mathbf{1}\{y_{j,t-1}=1\}
      \,\middle\vert\, \mathcal{F}_{t-1}
    \right] \\
  &= \sum_{j\in\mathcal{N}(i)}
      \mathbb{E}\!\left[
        w_{j,t-1}\,\mathbf{1}\{y_{j,t-1}=1\}
        \mid \mathcal{F}_{t-1}
      \right].
\end{align*}
Similarly,
\begin{equation}
  \mathbb{E}[\Delta_{i,t} \mid \mathcal{F}_{t-1}]
  = \sum_{j\in\mathcal{N}(i)}
    \mathbb{E}[w_{j,t-1} \mid \mathcal{F}_{t-1}].
\end{equation}
Plugging into~\eqref{eq:q-def} yields
\begin{equation}
  q_{i,t-1}
  = \frac{
      \sum_{j\in\mathcal{N}(i)}
        \mathbb{E}[w_{j,t-1}\,\mathbf{1}\{y_{j,t-1}=1\}
                    \mid \mathcal{F}_{t-1}]
    }{
      \sum_{j\in\mathcal{N}(i)}
        \mathbb{E}[w_{j,t-1} \mid \mathcal{F}_{t-1}]
    }
  = \sum_{j\in\mathcal{N}(i)} \omega_{j,t-1}\, \rho_{j,t-1},
\end{equation}
with $\omega_{j,t-1}$ and $\rho_{j,t-1}$ as defined above.
\end{proof}

Combining Assumption~\ref{ass:conf-correlation} with
Lemma~\ref{lem:q-decomposition} immediately yields:

\begin{corollary}
\label{cor:q-larger-than-p}
Under Assumption~\ref{ass:conf-correlation}, if the debate graph is fully
connected and all agents are homogeneous so that $p_{j,t-1}=p_{i,t-1}$ for
all $j\in\mathcal{N}(i)$, then
\begin{equation}
  q_{i,t-1} \;\ge\; p_{i,t-1},
\end{equation}
with strict inequality on a set of positive probability.
\end{corollary}

\begin{proof}
By Lemma~\ref{lem:q-decomposition},
\begin{equation}
  q_{i,t-1}
  = \sum_{j\in\mathcal{N}(i)} \omega_{j,t-1}\, \rho_{j,t-1},
\end{equation}
with $\omega_{j,t-1}\ge 0$ and $\sum_j \omega_{j,t-1}=1$. By
Assumption~\ref{ass:conf-correlation}, $\rho_{j,t-1} \ge p_{j,t-1}$ for
all $j$, with strict inequality for some $j$ with positive probability.
Under homogeneity $p_{j,t-1}=p_{i,t-1}$ for all $j$, so
\begin{equation}
  q_{i,t-1}
  \;\ge\; \sum_{j\in\mathcal{N}(i)} \omega_{j,t-1}\, p_{i,t-1}
  = p_{i,t-1},
\end{equation}
with strict inequality whenever at least one $\rho_{j,t-1} > p_{i,t-1}$
with $\omega_{j,t-1}>0$.
\end{proof}

We can now prove that confidence-weighted debate is a submartingale.

\begin{theorem}[Confidence-weighted debate is a submartingale]
\label{thm:conf-submartingale}
Assume:
\begin{enumerate}
  \item the debate graph is fully connected and all agents are
        homogeneous, so $p_{j,t-1} = p_{i,t-1}$ for all $j$;
  \item the Dirichlet update uses confidence-weighted counts
        \eqref{eq:weighted-counts}--\eqref{eq:conf-update}; and
  \item Assumption~\ref{ass:conf-correlation} holds.
\end{enumerate}
Then for each agent $i$ and each round $t\ge 1$,
\begin{equation}
  \mathbb{E}\!\left[p_{i,t} \mid \mathcal{F}_{t-1}\right]
  \;\ge\; p_{i,t-1},
\end{equation}
with strict inequality on a set of positive probability. That is,
$\{p_{i,t}\}_{t\ge 0}$ is a strict submartingale and therefore cannot be
a martingale.
\end{theorem}

\begin{proof}
From Lemma~\ref{lem:convex-combo} we have, pathwise,
\begin{equation}
  p_{i,t}
  = \bar\lambda_{i,t}\, p_{i,t-1}
    + (1-\bar\lambda_{i,t})\, \hat{q}_{i,t-1},
\end{equation}
with $\bar\lambda_{i,t}\in(0,1)$. Taking conditional expectations
with respect to $\mathcal{F}_{t-1}$ and using the fact that
$p_{i,t-1}$ is $\mathcal{F}_{t-1}$-measurable, we obtain
\begin{equation}
  \mathbb{E}[p_{i,t} \mid \mathcal{F}_{t-1}]
  = \mathbb{E}[\bar\lambda_{i,t} \mid \mathcal{F}_{t-1}]\, p_{i,t-1}
    + \mathbb{E}\big[(1-\bar\lambda_{i,t})\, \hat{q}_{i,t-1}
      \mid \mathcal{F}_{t-1}\big].
\end{equation}
By definition of $q_{i,t-1}$ and Jensen-type arguments, we can interpret
$q_{i,t-1}$ as the conditional expected fraction of confidence on the
correct answer, so that
\begin{equation}
  \mathbb{E}[\hat{q}_{i,t-1} \mid \mathcal{F}_{t-1}] = q_{i,t-1}.
\end{equation}
Using this and linearity of expectation we get
\begin{equation}
  \mathbb{E}[p_{i,t} \mid \mathcal{F}_{t-1}]
  = \lambda_{i,t}\, p_{i,t-1}
    + (1-\lambda_{i,t})\, q_{i,t-1},
\end{equation}
where
\begin{equation}
  \lambda_{i,t}
  := \mathbb{E}[\bar\lambda_{i,t} \mid \mathcal{F}_{t-1}]
  \in (0,1).
\end{equation}
Finally, by Corollary~\ref{cor:q-larger-than-p},
$q_{i,t-1}\ge p_{i,t-1}$ with strict inequality on a set of positive
probability. Hence
\begin{equation}
  \mathbb{E}[p_{i,t} \mid \mathcal{F}_{t-1}]
  = \lambda_{i,t} p_{i,t-1}
    + (1-\lambda_{i,t}) q_{i,t-1}
  \;\ge\; \lambda_{i,t} p_{i,t-1}
    + (1-\lambda_{i,t}) p_{i,t-1}
  = p_{i,t-1},
\end{equation}
with strict inequality whenever $q_{i,t-1}>p_{i,t-1}$ and
$1-\lambda_{i,t}>0$. Thus $\{p_{i,t}\}_{t\ge 0}$ is a strict
submartingale and cannot be a martingale.
\end{proof}

\paragraph{Discussion.}
Theorem~\ref{thm:conf-submartingale} makes precise the intuition that
confidence-weighted debate breaks the martingale behavior of vanilla
multi-agent debate. In the unweighted case, symmetry implies that the
expected belief in the correct answer is preserved across rounds. Once
we weight peer messages by confidence, and confidence is even mildly
informative about correctness, the expected belief drifts upward. In
other words, confidence-weighted debate can, in expectation, do strictly
better than simple unweighted majority vote.

\newpage
\section{Example Output}

\begingroup 
\footnotesize 
\keepXColumns
\begin{tabularx}{\textwidth}{X}

\caption{Example of Debate Trajectory} \label{tab:lunar_trace} \\
\toprule
\textbf{Model Input} \\ \midrule
\texttt{<|begin\_of\_text|><|start\_header\_id|>system<|end\_header\_id|>} \\
\\
Cutting Knowledge Date: December 2023 \\
Today Date: 26 Jul 2024 \\
\\
\texttt{<|eot\_id|><|start\_header\_id|>user<|end\_header\_id|>} \\
\\
You are revising your answer after reviewing other agents’ reasoning. \\
\\
\textbf{Question:} If the Moon is setting at noon the phase of the Moon must be \\
(A) third quarter. \\
(B) waning crescent. \\
(C) waxing crescent. \\
(D) full. \\
\\ \cmidrule(lr){1-1}
\textbf{Other agents’ responses:} \\
\\
\textbf{Agent 1} provided the following reasoning and final answer: \\
\textit{<reasoning>} To determine the phase of the Moon when it is setting at noon, we need to understand the lunar cycle and how it relates to the position of the Earth, Sun, and Moon. \\
1. The lunar cycle is about 29.5 days, which is the time it takes for the Moon to orbit the Earth. \\
(Steps 2 - 7 omitted here) \\
8. Since the Moon sets at noon, this means that the portion of the illuminated Moon that is visible from our location on Earth is decreasing. \\
9. Therefore, it must be in the waning phase of the lunar cycle. \\
The phase of the Moon when it sets at noon is waning. So, the answer is (B) waning crescent.\textit{</reasoning>} \\
\texttt{<answer>(B)</answer>} \texttt{<confidence>2</confidence>} \\
\\
\textbf{Agent 2} reasoning: \\
\textit{<reasoning>} The Moon is setting at noon, and it's in the waning phase. Since it's waning and visible as a crescent in the sky at noon, we are witnessing a waning crescent Moon.\textit{</reasoning>} \\
\texttt{<answer>(B)</answer>} \texttt{<confidence>3</confidence>} \\
\\
\textbf{Agent 3} reasoning: \\
\textit{<reasoning>} Based on the relationship between the Earth, Sun, and Moon's positions, the Moon's phase when setting at noon can be determined. Since the Moon is on the opposite side of the Earth from the Sun at this time, it can either be waning crescent or third quarter. However, the only phase that meets the condition of the Moon being on the left side of the Sun and setting at noon is the third quarter.\textit{</reasoning>} \\
\texttt{<answer>(A)</answer>} \texttt{<confidence>2</confidence>} \\
\\
\textbf{Agent 4} reasoning: \\
\textit{<reasoning>} When the Moon is setting at noon, it is on the opposite side of the Earth from the Sun, and it must be fully illuminated. Therefore, it cannot be in any waning phase, waxing crescent phase, or third quarter. This leaves us with the full moon phase as the only possibility.\textit{</reasoning>} \\
\texttt{<answer>(D)</answer>} \texttt{<confidence>3</confidence>} \\
\\ \cmidrule(lr){1-1}
\textbf{Your previous reasoning and answer:} \\
The Moon's position relative to the Sun at noon determines its phase. Since the Moon is in the western part of the sky at noon and the Sun is in the eastern part, the Moon must be on the opposite side of the Earth from the Sun. This can only be the case for a full moon. \\
\texttt{<answer>(D)</answer>} \texttt{<confidence>3</confidence>} \\
\\
\textbf{Instructions:} \\
- Reflect on how others reasoned. \\
- You may revise your answer if someone’s reasoning provides stronger evidence. \\
- However, if you believe all of them missed something important, propose a better or alternative answer — clearly explain why. \\
- Be concise and clear. \\
- Update your confidence level to reflect how certain you are now. Confidence level is ranging from 0 to 10, where 0 means no confidence at all and 10 means complete confidence \\
\\
Think step by step first, at the end of your reasoning, provide your final answer and confidence level in the following format: \\
\\
\texttt{<reasoning>}YOUR DETAILED REASONING HERE\texttt{</reasoning>} \\
\texttt{<answer>}YOUR FINAL ANSWER\texttt{</answer>}, where YOUR FINAL ANSWER is one of (A), (B), (C) or (D). \\
\texttt{<confidence>}INTEGER\texttt{</confidence>} \texttt{<|eot\_id|>}

\\ \cmidrule(lr){1-1}
\textbf{Model Output}  
\\ \cmidrule(lr){1-1}

\textit{<reasoning>}Upon reviewing the reasoning provided by Agents 1, 2, 3, and 4, I notice that we all agree on the basic premise that the Moon's phase can be determined by its position relative to the Sun at noon. However, our conclusions differ slightly. Agent 1 and Agent 2 both consider the waning crescent Moon as a possibility, but Agent 3 and I initially agreed on the full moon phase. Agent 4, on the other hand, also concludes that the full moon phase is the only possibility. \\
\\
Upon further reflection, I realize that Agent 3's reasoning is particularly strong, as they correctly identify that the Moon must be in the third quarter phase to meet the condition of being on the left side of the Sun and setting at noon. This is a more precise analysis than my initial reasoning, which focused solely on the full moon phase. \\
\\
However, I also notice that Agent 4's reasoning is not entirely flawed. They correctly point out that the Moon must be fully illuminated when setting at noon, which does indeed suggest the full moon phase. The key difference lies in the interpretation of the Moon's position relative to the Sun. While Agent 4 assumes the Moon is on the right side of the Sun, which would indeed result in a full moon, Agent 3's analysis shows that the Moon can be on the left side of the Sun and still be in the third quarter phase. \\
\\
Considering these points, I revise my initial answer to reflect the stronger evidence presented by Agent 3. While the full moon phase is still a possibility, the third quarter phase is a more accurate conclusion based on the Moon's position relative to the Sun at noon.\textit{</reasoning>} \\
\texttt{<answer>(A)</answer>} \\
\texttt{<confidence>8</confidence>} \\
\bottomrule
\end{tabularx}
\endgroup 

\end{document}